\documentclass[conference]{IEEEtran}
\usepackage[font=footnotesize]{caption}
\usepackage{subcaption}
\usepackage[compatibility=false]{caption}
\IEEEoverridecommandlockouts
\usepackage{cite}
\usepackage{pifont}   
\usepackage{amssymb}  
\newcommand{\cmark}{\ding{51}}                 
\newcommand{\xmark}{\ding{55}}                 
\newcommand{\pmark}{\ensuremath{\triangle}}  
\usepackage[english]{babel}
\usepackage[autostyle, english = american]{csquotes}
\MakeOuterQuote{"}
\usepackage{amsmath,amssymb,amsfonts}
\usepackage{graphicx}
\usepackage{textcomp}
\usepackage{adjustbox}
\usepackage{lscape}
\usepackage{hologo}
\usepackage{placeins}
\usepackage[utf8]{inputenc}
\usepackage{float}
\usepackage[linesnumbered,ruled,vlined]{algorithm2e}
\usepackage{booktabs}
\usepackage{algpseudocode}
\usepackage{tabularx}
\usepackage{comment}
\usepackage{xcolor}
\usepackage[breaklinks=true, colorlinks, bookmarks=true]{hyperref}
\usepackage{cleveref} 
\usepackage[font=footnotesize]{caption}
\usepackage{subcaption}
\usepackage{amsthm}
\usepackage[breaklinks=true, colorlinks, bookmarks=true]{hyperref}
\usepackage{tikz}
\usepackage{pgfplots}
\pgfplotsset{compat=1.9}
\usepackage{booktabs}
\usepackage{siunitx}
\newtheorem{theorem}{Theorem}[section]
\newtheorem{definition}[theorem]{Definition}
\newtheorem{lemma}[theorem]{Lemma}
\usepackage{tabularx,booktabs,makecell,array}
\newcolumntype{Y}{>{\raggedright\arraybackslash}X} 
\newcolumntype{C}{>{\centering\arraybackslash}p{0.14\columnwidth}} 

\IEEEoverridecommandlockouts  
\def\BibTeX{{\rm B\kern-.05em{\sc i\kern-.025em b}\kern-.08em
    T\kern-.1667em\lower.7ex\hbox{E}\kern-.125emX}}
\begin{document}

\title{\LARGE \bf Bootstrapped LLM Semantics for Context-Aware Path Planning\

\thanks{The presented work has been supported by the U.S. National Science Foundation (NSF) CAREER Award through grant No. 2047138. The authors gratefully acknowledge the support from the NSF. Any opinions, findings, conclusions, and recommendations expressed in this paper are those of the authors and do not necessarily represent those of the NSF.\\
$^1$Department of Civil, Construction, and Environmental Engineering, San Diego State University, San Diego, CA, United States\\
$^2$Department of Electrical and Computer Engineering, University of California San Diego, San Diego, CA, United States\\
$^3$Department of Computer Science, San Diego State University, San Diego, CA, United States}
}

\author{{ Mani Amani$^{1 , 2}$, Behrad Behesthi$^{3}$, Reza Akhavian$^{1}$}}

\maketitle

\begin{abstract}
Prompting robots with natural language (NL) has largely been studied as \emph{what} task to execute (goal selection, skill sequencing) rather than \emph{how} to execute that task safely and efficiently in semantically rich, human-centric spaces. We address this gap with a framework that turns a large language model (LLM) into a stochastic \emph{semantic sensor} whose outputs modulate a classical planner. Given a prompt and a semantic map, we draw multiple LLM “danger” judgments and apply a Bayesian bootstrap to approximate a posterior over per-class risk. Using statistics from the posterior, we create a potential cost to formulate a path planning problem. Across simulated environments and a BIM-backed digital twin, our method adapts \emph{how} the robot moves in response to explicit prompts and implicit contextual information. We present qualitative and quantitative results.
\end{abstract}

\section{Introduction}
Robots operating in human-centric, cluttered environments must plan not only for geometric safety  \cite{amani2024intelligent} but also for \emph{semantic} context: which regions are likely to be busy, which objects are currently in use (where robot motion would be disruptive), and which hazards merit conservative treatment. While classical planners encode proximity via distance-based potentials or minimum distance fields (MDF), they are agnostic to evolving, task-level semantics that are naturally conveyed in language \cite{APF,Voxblox}.

LLMs exhibit strong reasoning capabilities \cite{LLMreason}, and many works use them to have robots act on NL commands \cite{LTLCODEGEN,voxposer}. With the rise of LLMs, Vision–Language (VLM) and Vision–Language–Action (VLA) models have also gained attention for their impressive grounding and decision-making abilities. Recent research in semantic mapping and cost shaping seeks to bridge geometric planning and higher-level context. Semantic SLAM augments maps with object labels and affordances \cite{semantic_slam}, enabling robots to avoid “chairs” or “tables” rather than abstract occupancy. Language as costs (LaC) \cite{langascost}, uses "anxiety scores" from vision–language models to shape paths. However, LaC depends on online visual grounding or assumes static language-to-cost mappings, limiting adaptation to evolving context through user commands. Parallel efforts in risk-sensitive planning, such as chance-constrained planning \cite{chance_constrained}, Conditional Value-at-Risk (CVaR)-based path planning \cite{risk_aware_planning}, and robust RL \cite{robust_planning}, account for uncertainty in dynamics or perception but rarely incorporate natural language as a source of semantic uncertainty. Conversely, language-grounded planners (e.g., LM-Nav \cite{LMNav} and LTLCodeGen \cite{LTLCODEGEN}) leverage LLMs to translate prompts into automata or plans. Yet once an automaton is synthesized, it is essentially fixed and prompt-/task-specific. VoxPoser \cite{voxposer} generates dense, language-conditioned 3D voxel value maps for 6-DoF manipulation tasks. However, the reported results are limited to small tabletop setups specifically for robot arm manipulation. 

To address this gap, we propose a language-to-cost formulation that produces a reusable, task-agnostic cost map that can be optimized across many problems, enabling multiple levels of planning and understanding while remaining agnostic to the underlying framework. By treating the LLM as a stochastic semantic sensor, we explicitly model its uncertainty rather than relying on single-shot outputs or deterministic parses. The Bayesian bootstrap \cite{bayesianbootstrap} provides a nonparametric posterior over semantic risk judgments, and CVaR furnishes a principled mechanism to tune conservatism during planning. This enables a closed-form integration of prompt-conditioned semantics into classical potential–field–based costs \cite{amani2024}, preserving optimality guarantees while remaining adaptable to evolving NL inputs. In addition, our framework supports both explicit and implicit user prompts. Beyond direct instructions, it can interpret contextual cues and align them with environmental semantics, yielding qualitatively better path planning decisions. Table~\ref{tab:lang_map_closed_form_implicit} highlights how our method compares to prior work in terms of language-to-map grounding, closed-form formulations, user prompting, and implicit context understanding.
\begin{table}
  \centering
  \small
  \setlength{\tabcolsep}{3pt}
  \renewcommand{\arraystretch}{1.05}
  \caption{Related works (\cmark present, \xmark absent, \pmark partial).}
  \label{tab:lang_map_closed_form_implicit}
  
  \begin{tabularx}{\columnwidth}{@{}YCCCC@{}}
    \toprule
    Method & Language to Map & Closed-form & User-prompted & Implicit Context \\
    \midrule
    \makecell[l]{Ours}                   & \cmark & \cmark & \cmark & \cmark \\
    VoxPoser \cite{voxposer}             & \cmark & \xmark & \cmark & \cmark \\
    \makecell[l]{LM-Nav \cite{LMNav}}    & \xmark & \xmark & \cmark & \xmark \\
    \makecell[l]{LTLCodeGen \cite{LTLCODEGEN}} & \xmark & \xmark & \cmark & \pmark \\
    LaC \cite{langascost}                & \cmark & \cmark & \xmark & \pmark \\
    \bottomrule
  \end{tabularx}
\end{table}

\par
\par
Specifically, the following are the main contributions of this paper:
\begin{itemize}
  \item A nonparametric \emph{distribution bootstrapping} view of LLM outputs for estimation: treating the LLM as a stochastic sensor and using Bayesian bootstrap to capture semantic uncertainty without a fixed parametric prior.
  \item Formulating a semantic-risk aware pathfinding into a Deterministic Shortest Path (DSP) problem with optimality. This approach can be generalized for any cost optimization problem. 
  \item A context-sensitive cost shaping that maps the bootstrap posterior to planning costs via CVaR, enabling tunable conservatism and per-class, prompt-conditioned repulsive potentials that plug into classical planners.
\end{itemize}
The overall methodology can be seen in Figure \ref{fig:mainfig}.
\begin{figure*}[ht]
    \centering
    \includegraphics[width=0.85\linewidth]{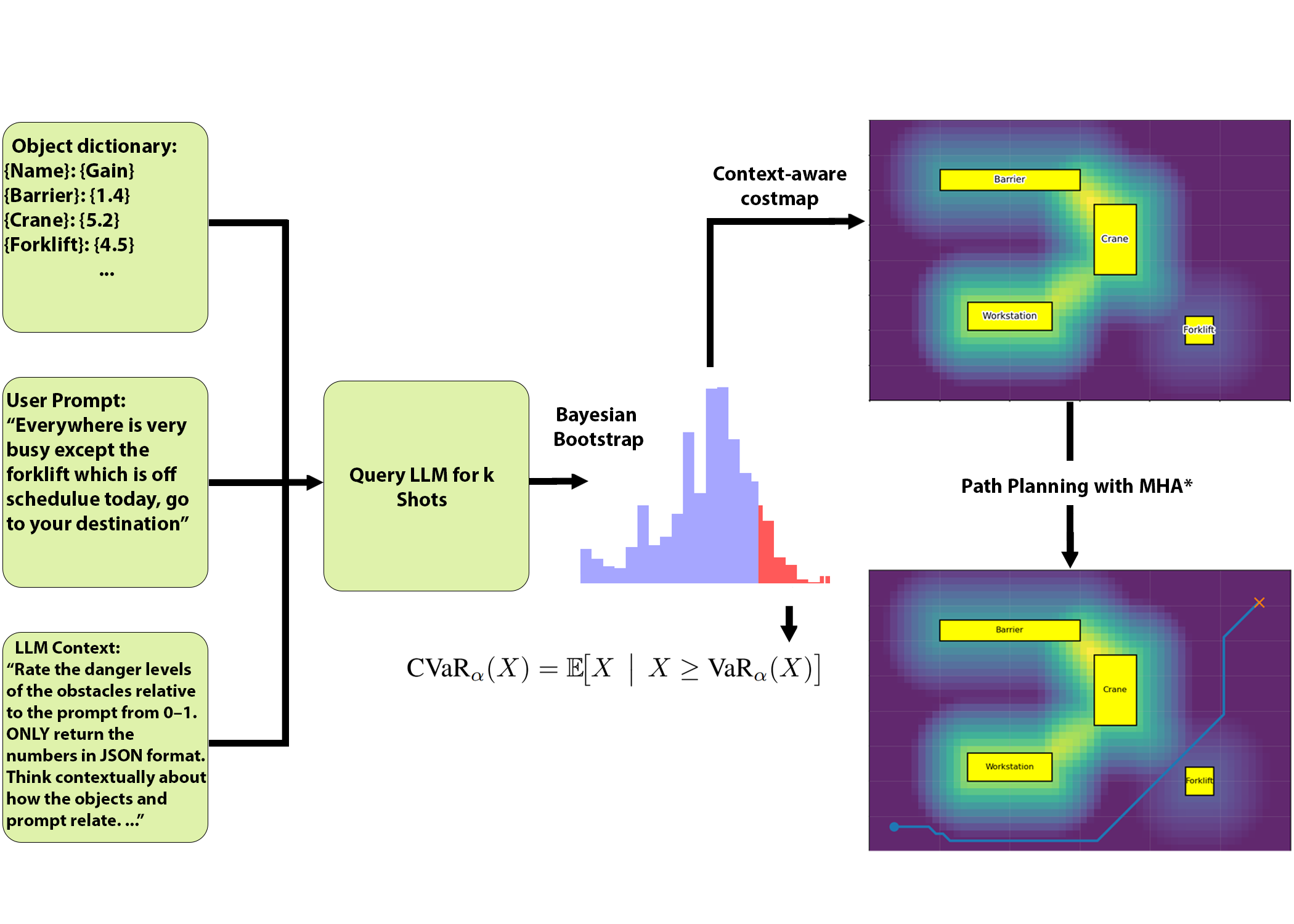}
    \caption{Overall framework to support the context-aware understanding of the pathplanning problem}
    \label{fig:mainfig}
\end{figure*}

\section{Problem Setup}
We interpret the LLM as a stochastic sensor that produces independent and identically distributed (i.i.d.) readings 
in the bounded interval $[0,1]$. Let $\mathcal{X} = \{x_1, x_2, \ldots, x_k\}$ denote a sample of such 
readings. Rather than using individual values, we employ $\mathcal{X}$ to approximate the underlying 
distribution of the sensor. This distribution is then used as a proxy to model the output distribution 
of a large language model (LLM) under a given prompt $P$.

The intuition is that the stochastic sensor acts as a \emph{semantic regressor}, mapping prompts $P$ to 
avoidance sentiments over an obstacle set $\mathcal{O} = \{o_1, o_2, \ldots, o_m\}$ for which the semantic label of each obstacle is known. For each obstacle label
$o \in \mathcal{O}$, the inferred sentiment distribution is denoted by
\begin{equation}
    \pi(o \mid P, \mathcal{X})
\end{equation}
which captures the predicted danger level conditioned on both the prompt and the sensor samples.

We then embed these sentiment-driven distributions into a path planning framework. Specifically, we 
formulate a DSP problem, where the edge cost is a combination of the repulsive cost and the distance cost.

\section{Mathematical Framework}
\subsection{Repulsive Fields}
Distance-based repulsive fields are a common approach to formalizing obstacle avoidance in robotic 
navigation tasks. Techniques such as signed distance fields (SDF), Euclidean distance fields (EDF), 
and other variants have been widely employed to quantify the spatial presence of objects within an 
environment. Classical potential field methods generate repulsive forces that penalize trajectories 
leading the robot too close to obstacles, thereby reducing the likelihood of collision.
In this work, we adopt a similar potential field formulation, where an exponential penalty function 
discourages proximity to nearby obstacles \cite{amani2024,victor2022dynamical}. A representative candidate function is given by
\begin{equation}
    \phi(d) = \lambda \exp\!\left(-{d}\right),
\end{equation}
where $d$ denotes the minimum distance to some obstacle, $\lambda > 0$ controls the strength of the repulsion. This formulation ensures that obstacles exert strong 
penalties at close range while diminishing influence at larger distances, thereby encouraging safe 
navigation without overly restricting global motion. Note that each obstacle \(o_i\) can have its own specific \(\lambda_i\) value depending either on user input, or on the semantic value. 
\subsection{Semantic Augmentation of Repulsive Fields}
To incorporate semantic information into the repulsive field, we require a scalar quantity derived 
from the stochastic sensor distribution. Since the LLM-based sensor produces i.i.d. samples conditioned 
on a prompt $P$, we seek to approximate the underlying distribution and extract a statistic that can 
be used to modulate the repulsive field strength.

We employ Bayesian bootstrapping to approximate this distribution. Concretely, let 
$\mathcal{X} = \{x_1, x_2, \ldots, x_k\}$ denote $K$ i.i.d. samples obtained from the LLM sensor output. 
Instead of assuming a fixed parametric prior, we draw $R$ resamples by generating Dirichlet-distributed 
weights $\mathbf{w}^{(r)} \sim \mathrm{Dirichlet}(1, \ldots, 1)$ and forming weighted empirical 
distributions
\begin{equation}
    F^{(r)}(x) = \sum_{i=1}^k w^{(r)}_i \, \delta_{x_i}(x), \quad r = 1, \ldots, R,
\end{equation}
where $\delta_{x_i}$ is the Dirac mass at $x_i$. Each $F^{(r)}$ provides a plausible realization of the posterior distribution of the LLM sensor. 

From these bootstrap distributions, we can compute a scalar statistic 
(e.g., the mean or a risk-sensitive quantile), denoted by $\hat{s}$, which serves as the semantic 
augmentation factor.
This approach allows the repulsive field to be adaptively modulated by semantic danger signals inferred 
from the LLM outputs, seamlessly combining geometric proximity with language-model-informed avoidance 
criteria.
\par
The advantage of approximating the distribution of the LLM is that we can now use certain statistics to influence the mission planning. We choose to use the CVaR \cite{CVAR} given by
\begin{align}
\text{CVaR}_{\alpha}(\mathcal{X})= \mathbb{E}\!\bigl[\mathcal{X} \,\bigm|\, \mathcal{X} \ge \text{VaR}_{\alpha}(\mathcal{X})\bigr],\\
\text{VaR}_{\alpha}(\mathcal{X}) =  \inf\{\,x\in\mathbb{R} : F(x)\ge \alpha\,\}.
\end{align}

to scale the potential function give by \(\lambda_{\text{scaled}} = \lambda_{\text{prior}}.\text{CVaR}_{\alpha}(\mathcal{X})\).
\subsection{Multi-Heuristic Path Planning}

We minimize a combined edge cost 
\begin{equation}
c(s,s') = d(s,s') + \gamma\,\Phi(s')
\label{eq:edgecost}
\end{equation}
where \(d(s,s')\) is the nominal Euclidean distance cost and \(\Phi:\mathcal{S}\to\mathbb{R}_{\ge0}\) is a nonnegative repulsive potential defined on states. Let \(\mathcal{O}\) denote the set of obstacles and \(d_i(s)\) the Euclidean distance from \(s\) to obstacle \(i\). $\gamma \geq 0$ is a tunable scaling factor, and 
$\Phi$ is the obstacle avoidance cost. The value of $\Phi$ is influenced by the predicted danger 
distribution $\pi(o \mid P, \mathcal{X})$, thereby coupling geometric distance with language-model-informed 
semantic risk.  We define
\begin{equation}
  \Phi(s) \;=\; \sum_{i\in\mathcal{O}} \phi\!\big(d_i(s)\big),
\end{equation}
with \(\phi:\mathbb{R}_{\ge0}\to\mathbb{R}_{\ge0}\) decreasing (e.g., \(\phi(r)=\lambda\,e^{-r}\)) with scale \(\lambda>0\). The sum helps us create an additive cost map that would penalize cluttered areas if need be. This sum is relatively lightweight since it can be done in $O(cN)$ where $N$ is the number of grids and $c$ is the number of classes.

We run MHA* \cite{MHA*} with a single consistent \emph{anchor} heuristic and one (informative, possibly inadmissible) auxiliary heuristic that incorporates obstacle proximity.
\begin{definition}[Consistent heuristic]
Let $h:S\!\to\!\mathbb{R}_{\ge 0}$ be a heuristic on states $S$ with goal set $G$
and step costs $c(n,n')\!\ge\!0$ for each edge $(n,n')$.
The heuristic $h$ is \emph{consistent} if and only if
\[
  h(n) \le c(n,n') + h(n') \quad \text{for all edges }(n,n')
\]
and
\[
  h(g)=0 \quad \text{for all } g\in G .
\]
Equivalently, $h$ satisfies the triangle inequality on edges and is admissible.
\end{definition}
\begin{lemma}[Euclidean anchor is consistent for \(c\)]
Assume \(d(s,s') \ge \|x_s - x_{s'}\|_2\) for all neighbors (e.g., 4/8-connected grids with costs \(1,\sqrt{2}\)). Then the Euclidean distance
\(
  h_0(s)=\|x_s-x_g\|_2
\)
is consistent (hence admissible) for the combined cost \(c(s,s')=d(s,s')+\gamma\,\Phi(s')\).
\end{lemma}
\begin{proof}
For any edge \((s,s')\), by the triangle inequality
\(
h_0(s) \le \|x_s-x_{s'}\|_2 + h_0(s') \le d(s,s') + h_0(s').
\)
Since \(\gamma\,\Phi(s')\ge 0\), we have
\(
h_0(s) \le d(s,s') + \gamma\,\Phi(s') + h_0(s') = c(s,s') + h_0(s').
\)
Thus \(h_0\) is consistent w.r.t.\ \(c\).
\end{proof}

\paragraph{Auxiliary heuristic (informative, may be inadmissible).}
To anticipate repulsive cost along a straight path to the goal, we approximate a line integral of \(\Phi\) from \(s\) to \(g\). Let \(L=\lceil \|x_s-x_g\|_2/\Delta \ell \rceil\) samples
\(
x_\ell = x_s + \tfrac{\ell}{L}(x_g-x_s), \; \ell=1,\dots,L,
\)
with step \(\Delta \ell>0\). Let \(\delta_\ell = \min_{i\in\mathcal{O}} d_i(x_\ell)\) be the clearance at sample \(x_\ell\).
Using \(\phi(r)=\lambda\,e^{-r}\) yields
\begin{equation}
\Phi(x_\ell) \;=\; \sum_{i\in\mathcal{O}} \phi\!\big(d_i(x_\ell)\big)
\;\approx\; \lambda_{\ell}\,e^{-\delta_\ell},
\end{equation}
when obstacles are dense and the nearest dominates. The auxiliary heuristic is then
\begin{equation}
h_{\text{aux}}(s)
\;\approx\; \gamma \sum_{\ell=1}^{L} \Phi(x_\ell)\,\Delta\ell
\;=\; \gamma \sum_{\ell=1}^{L} \lambda_{\ell}\,e^{-\delta_\ell}\,\Delta\ell,
\end{equation}
a Riemann-sum approximation to
\(
\gamma\int_0^{\|x_s-x_g\|_2}\!\!\Phi\big(x_s + t\,\hat{v}\big)\,dt
\)
with \(\hat{v}=(x_g-x_s)/\|x_g-x_s\|_2\).
This heuristic is generally inadmissible but can be highly informative near clutter. The intuition is to use the sum of all potentials from a line from the node directly to the goal to avoid highly dense environments. MHA* uses \(h_0\) as the anchor; therefore, the returned solution retains the anchor’s optimality. The auxiliary queue can accelerate search but cannot worsen that bound; in the worst case, it only adds overhead in expansions \cite{MHA*}. Note that when \(\gamma=0\), the problem reduces to the traditional A* problem using Euclidean distance as the heuristic and the \(g\)-cost. We use the commonly used A* pathfinding algorithm as one of the baselines in our experiments.
\section{Experiments and Results}
We experiment with several different simulation scenarios using different use prompts and space configurations for the robot to conduct path planning.  Although any semantic map can drive our planner, we use an industrial construction-site example to stress-test contextual reasoning in simulation. For the real-world proof of concept, we deploy the system in a university lab environment.

\subsection{Scenario 1, explicit object avoidance}
We experiment with an environment with semantically known maps. In this environment, we experiment with 3 separate prompts:
\begin{enumerate}
    \item "The workzone is very busy today, go to your destination."
    \item "The workzone is very empty today, go to your destination."
    \item "Every station is very busy except the forklift, which is off schedule today. Go to your destination."
\end{enumerate}
The first two prompts show clear, explicit, environment-level danger sentiments, and the third one shows object-specific sentiment.  
\par
For all of these prompts we used the hyperparameters: \(R = 3000\), \(k = 16\), \(\tau = 1.0\) on the GPT-4o-mini LLM through the OpenAI API. We used seed = 7 for drawing from the Dirichlet distribution. The LLM is instructed to return only numbers between 0-1 as a rating for the danger levels that it has perceived in a JSON format. We compare our method with two baselines. One baseline is the A* algorithm that minimizes Euclidean path length \cite{AStar}, a widely used approach to robotic path planning. Comparison to LaC. We benchmark against LaC \cite{langascost}, which maps language-derived anxiety scores into planning costs in real time. For a controlled comparison, both methods plan in the same known environment; LaC is provided the environment map and applies the cost mapping described by its authors, with anxiety scores supplied from the same sources as ours. This removes mapping uncertainty and isolates differences in planning behavior. In contrast, our framework additionally accepts test-time natural-language prompts and contextual metadata that modulate semantic risk. When prompts indicate low risk, our planner relaxes conservatism, yielding shorter paths at similar success. When prompts indicate higher risk, it increases conservatism, producing larger average and minimum obstacle clearances with only modest path stretch. Across our evaluations, relative to LaC under the same environmental knowledge, we observe higher clearances and, in low-risk settings, improved efficiency while maintaining success rates. Qualitative results for the busy, empty, and object-specific prompts appear in Figure~\ref{fig:Very_Busy_Results}, Figure~\ref{fig:very-empty-results}, and Figure~\ref{fig:Forklift_Results}; aggregate path and posterior statistics are reported in Tables~\ref{tab:prompt1-metrics},~\ref{tab:prompt2-metrics}, and~\ref{tab:prompt3-metrics}.
\begin{figure}
    \centering
    \includegraphics[width=1.0\linewidth]{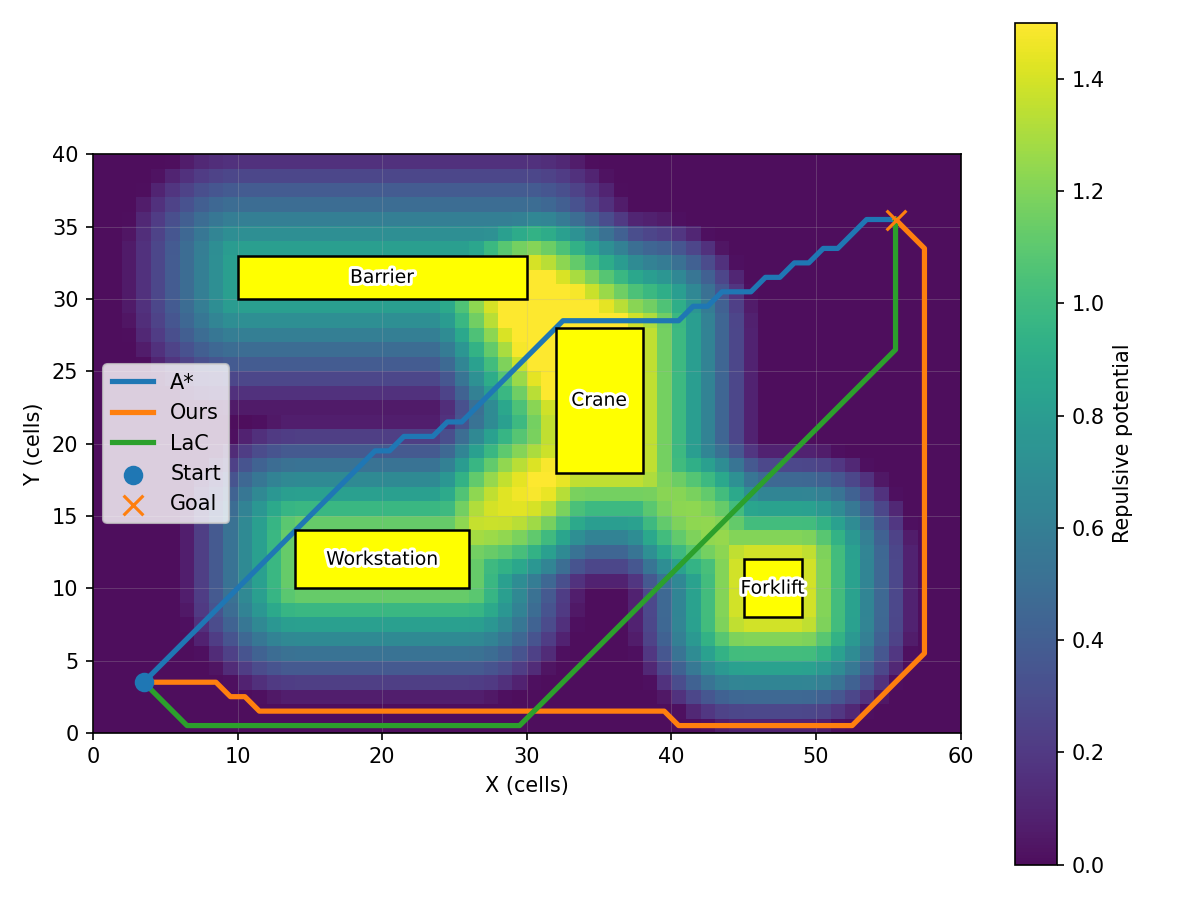}
    \caption{Overlay of paths given the prompt implying that the workzone is very busy so a more risk-aware path should be taken.}
    \label{fig:Very_Busy_Results}
\end{figure}

\begin{table}
\setlength{\tabcolsep}{4pt}
\centering
\footnotesize
\caption{Bayesian Bootstrap Diagnostics and Path Metrics for Prompt~1}
\label{tab:prompt1-metrics}

\begin{tabular}{l *{4}{S[table-format=1.3]} c}
\toprule
\multicolumn{6}{c}{\textbf{Hyperparameters} ($k{=}16$, $R{=}3000$, $\alpha=0.1$, $\tau{=} 1.0$, $\gamma = 1.5$)} \\
\midrule
& {\bfseries Workstation} & {\bfseries Crane} & {\bfseries Barrier} & {\bfseries Forklift} & \\
\cmidrule(lr){2-5}
\textbf{Posterior CVaR} & 0.600 & 0.720 & 0.437 & 0.741 & \\
\midrule
\multicolumn{6}{c}{\textbf{Path Metrics}} \\
\midrule
\textbf{Method} & \textbf{Length} & \textbf{Min.\ Dist.} & \textbf{Avg.\ Dist.} & \multicolumn{2}{c}{} \\
\midrule
A* path       & 65.25 & 0.50 & 5.83 & \multicolumn{2}{c}{} \\
Ours & 88.14   & 7.50  & 10.43 & \multicolumn{2}{c}{} \\
LaC & 73.01 & 2.92 &  10.02 & \multicolumn{2}{c}{}  \\ 
\bottomrule
\end{tabular}
\end{table}

\begin{figure}
    \centering
    \includegraphics[width=1.0\linewidth]{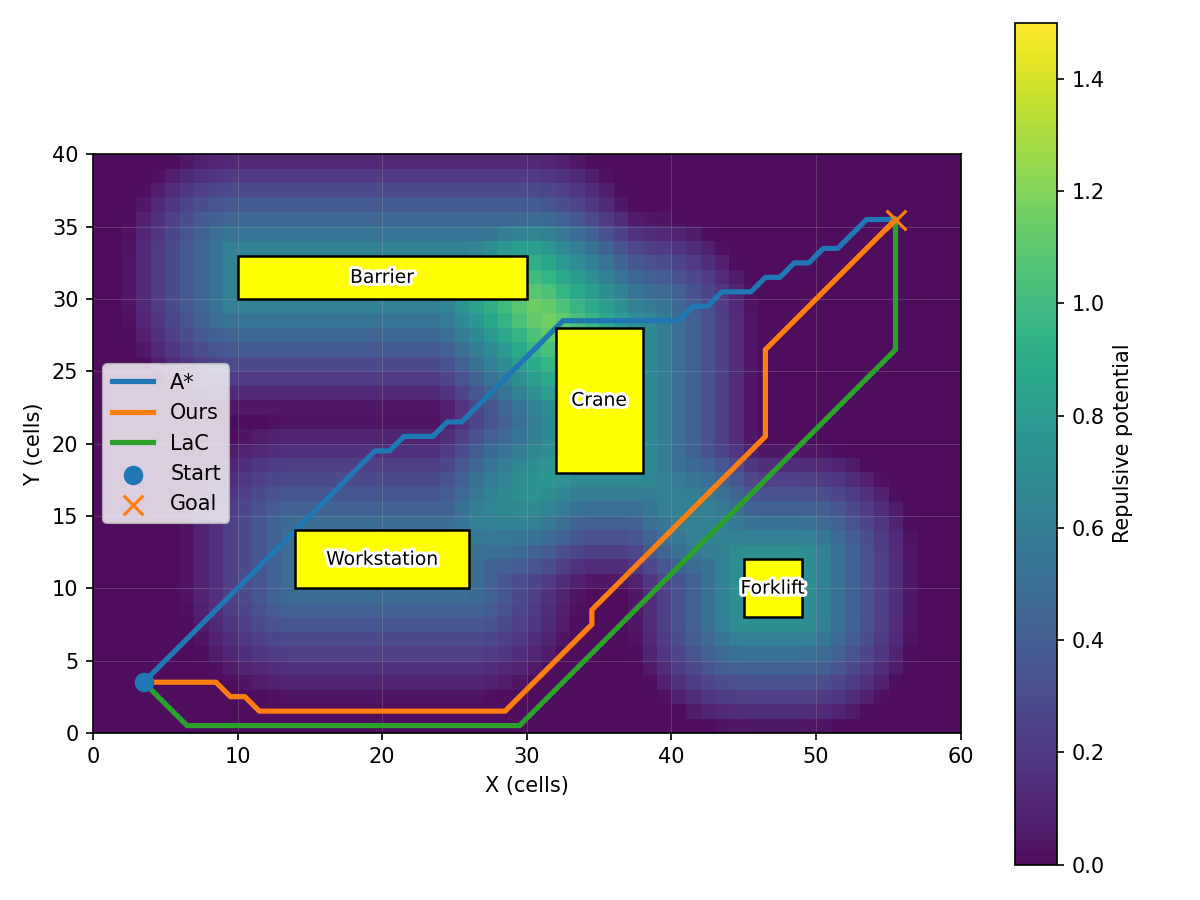}
    \caption{Overlay of paths given a prompt implying the work zone is very empty, so a more risk-neutral path should be taken.}
    \label{fig:very-empty-results}
\end{figure}

\begin{table}
\setlength{\tabcolsep}{4pt}
\centering
\footnotesize
\caption{Bayesian Bootstrap Diagnostics and Path Metrics for Prompt~2}
\label{tab:prompt2-metrics}

\begin{tabular}{l *{4}{S[table-format=2.3]} c}
\toprule
\multicolumn{6}{c}{\textbf{Hyperparameters} ($k{=}16$, $R{=}3000$, $\alpha=0.1$, $\tau{=} 1.0$, $\gamma = 1.5$)}\\
\midrule
& {\bfseries Workstation} & {\bfseries Crane} & {\bfseries Barrier} & {\bfseries Forklift} & \\
\cmidrule(lr){2-5}
\textbf{Posterior CVaR} & 0.259 & 0.365 & 0.332 & 0.376 & \\
\midrule
\multicolumn{6}{c}{\textbf{Path Metrics}} \\
\midrule
\textbf{Method} & \textbf{Length} & \textbf{Min.\ Dist.} & \textbf{Avg.\ Dist.} & \multicolumn{2}{c}{} \\
\midrule
A* path& 65.25 & 0.50 & 5.83 & \multicolumn{2}{c}{} \\
Ours    & 71.01 & 4.30 & 8.95 & \multicolumn{2}{c}{} \\
LaC & 73.01 & 2.92 &  10.02 & \multicolumn{2}{c}{}  \\ 
\bottomrule
\end{tabular}
\end{table}

\begin{figure}
    \centering
    \includegraphics[width=1.0\linewidth]{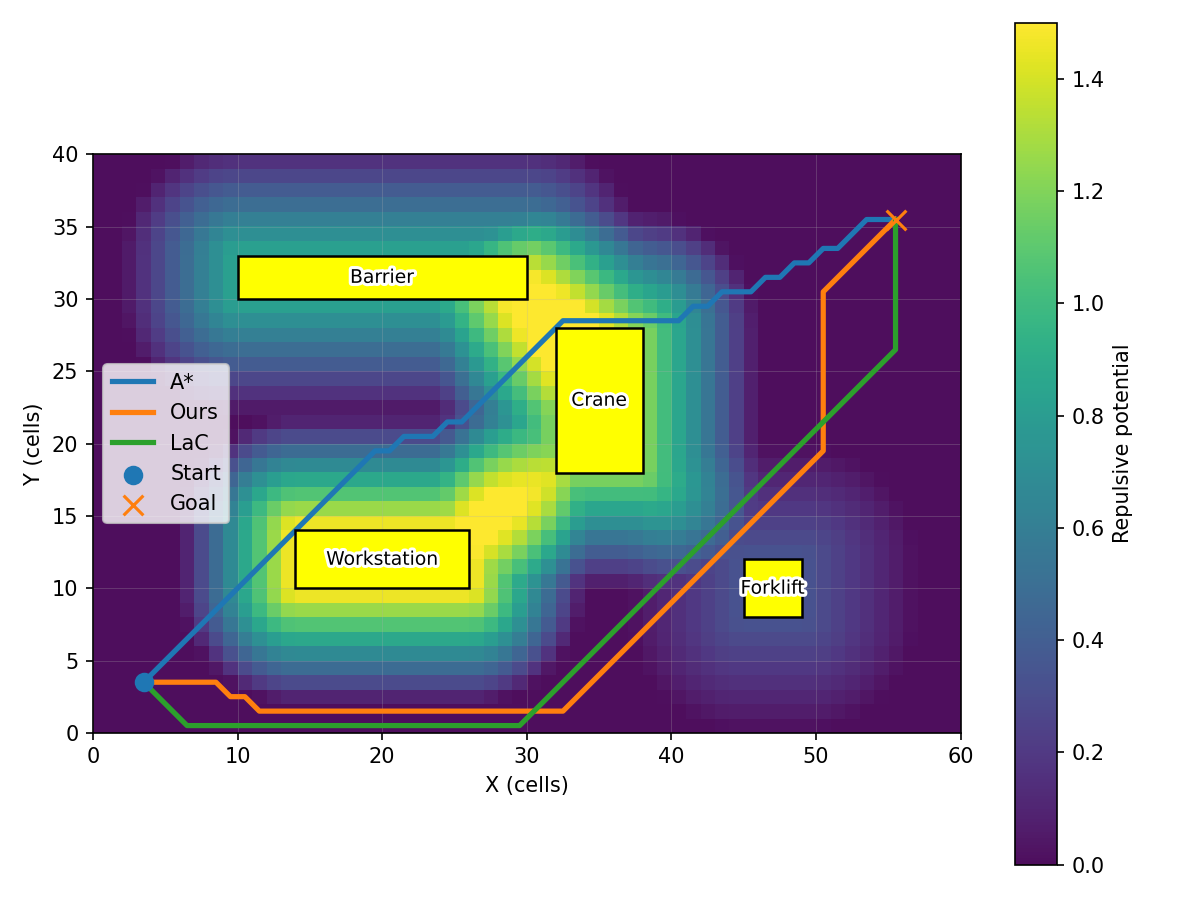}
    \caption{Overlay of paths given the prompt implying that the forklift is empty and a path in that direction should be taken. Note that the potential field around the forklift is dimmer than the rest of the obstacles. This shows the efficacy of the framework in updating semantic costs depending on obstacle and prompt qualities.}
    \label{fig:Forklift_Results}
\end{figure}

\begin{table}
\setlength{\tabcolsep}{4pt}
\centering
\footnotesize
\caption{Bayesian Bootstrap Diagnostics and Path Metrics for Prompt~3}
\label{tab:prompt3-metrics}

\begin{tabular}{l *{4}{S[table-format=2.3]} c}
\toprule
\multicolumn{6}{c}{\textbf{Hyperparameters} ($k{=}16$, $R{=}3000$, $\alpha=0.1$, $\tau{=} 1.0$, $\gamma = 1.5$)} \\
\midrule
& {\bfseries Workstation} & {\bfseries Crane} & {\bfseries Barrier} & {\bfseries Forklift} & \\
\cmidrule(lr){2-5}
\textbf{Posterior CVaR} & 0.776 & 0.626 & 0.440 & 0.177 & \\
\midrule
\multicolumn{6}{c}{\textbf{Path Metrics}} \\
\midrule
\textbf{Method} & \textbf{Length} & \textbf{Min.\ Dist.} & \textbf{Avg.\ Dist.} & \multicolumn{2}{c}{} \\
\midrule
A* path & 65.25 & 0.50 & 5.83 & \multicolumn{2}{c}{} \\
Ours    & 73.36 & 1.58 & 8.93 & \multicolumn{2}{c}{} \\
LaC & 73.01 & 2.92 &  10.42 & \multicolumn{2}{c}{}  \\ 
\bottomrule
\end{tabular}
\end{table}

\subsection{Scenario 2 implicit object avoidance.}

Here, we experiment with two different prompts for a new environment to test if the model can infer correctly with subtextual and contextual information regarding the ongoing activities.
\begin{enumerate}
    \item "The plumbing team just started work, go to the destination."
    \item "The plumbing team just finished their shift and left, go to the destination."
\end{enumerate}
We experiment with these two prompts to measure how effectively this information conveys the potential activities and semantics. By querying the LLM that the plumbing team is here to do work, we implicitly mention that there is a higher probability of activity, that the pipe storage can be more dynamic, and have a higher risk of collision and disruption. Our experiments show that the LLM is able to accurately weight these prompts to the environment to give back semantically correct paths. 
\par
In the second scenario, we see the same behavior that we saw in the first when comparing LaC and A* to our framework. Our framework is able to understand the implicit context that exists within the user prompt, whether the environment and pipes are active or not, and act upon that. However, both LaC and A* are unable to consider that. The resulting path for the "starting shift" prompt can be seen in Figure \ref{fig:plumbin}, and for the "ending shift" can be seen in Figure \ref{fig:Plumbout_Results}. The path and posterior statistics can be seen in Table \ref{tab:prompt-implicit-plumbing-active} and Table \ref{tab:prompt3-plumbout}.
\par
\begin{figure}
    \centering
    \includegraphics[width=1.0\linewidth]{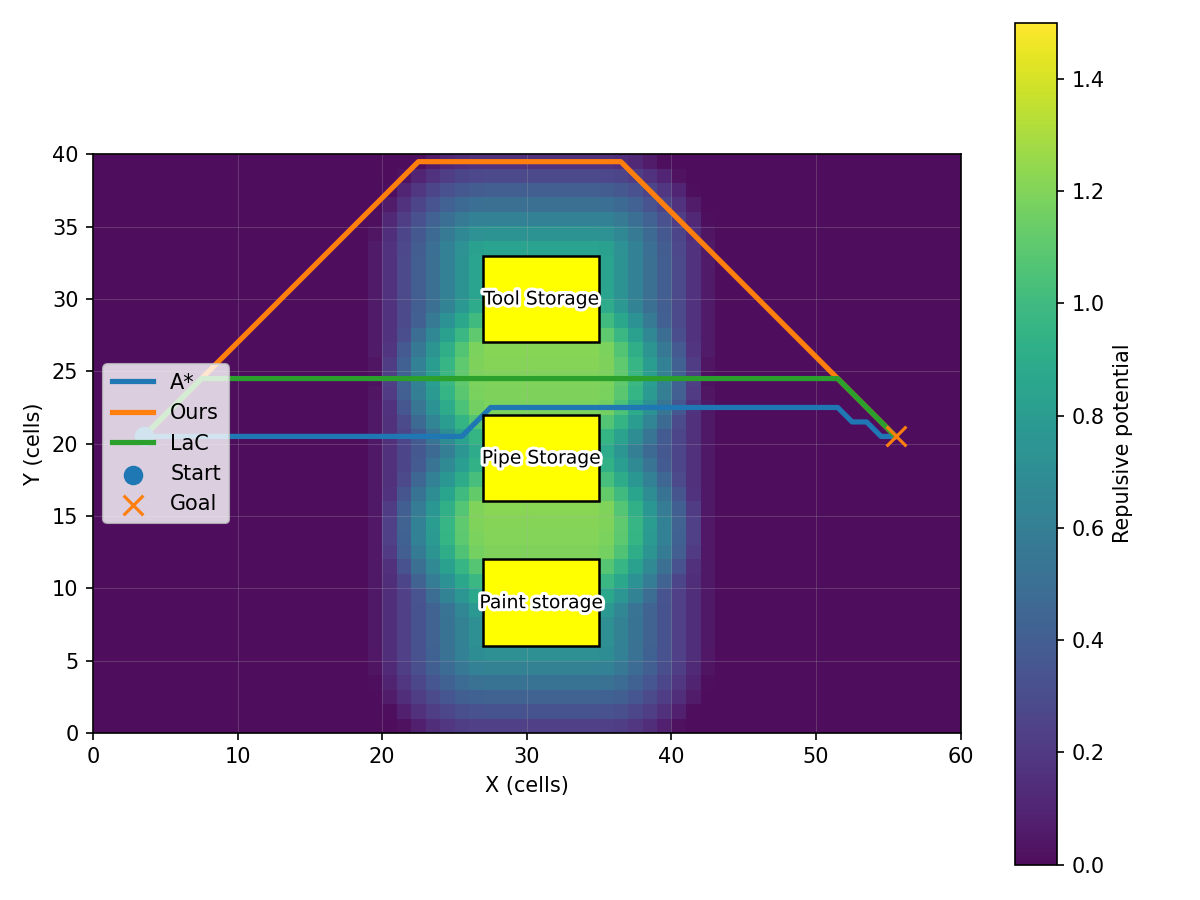}
    \caption{Prompt notifying the robot that the plumbing team is in, implicitly implying that the plumbing storage unit might be active}
    \label{fig:plumbin}
\end{figure}

\begin{table}
\setlength{\tabcolsep}{4pt}
\centering
\footnotesize
\caption{Bayesian Bootstrap Diagnostics and Path Metrics for Implicit Prompt with Plumbing active}
\label{tab:prompt-implicit-plumbing-active}

\begin{tabular}{l *{3}{S[table-format=2.3]} cc}
\toprule
\multicolumn{6}{c}{\textbf{Hyperparameters} ($k{=}16$, $R{=}3000$, $\alpha=0.1$, $\tau{=} 1.0$, $\gamma = 1.0$)} \\
\midrule
& {\bfseries Paint Storage} & {\bfseries Pipe Storage} & {\bfseries Tool Storage} & & \\
\cmidrule(lr){2-4}
\textbf{Posterior CVaR} & 0.360 & 0.528 & 0.442 & & \\
\midrule
\multicolumn{6}{c}{\textbf{Path Metrics}} \\
\midrule
\textbf{Method} & \textbf{Length} & \textbf{Min.\ Dist.} & \textbf{Avg.\ Dist.} & \multicolumn{2}{c}{} \\
\midrule
A* path & 53.66 & 0.50 & 9.44  & \multicolumn{2}{c}{} \\
Ours      & 67.74 & 5.70 & 11.66 & \multicolumn{2}{c}{} \\
LaC & 55.31 & 2.50 &  10.72 & \multicolumn{2}{c}{}  \\ 
\bottomrule
\end{tabular}
\end{table}


\begin{figure}
    \centering
    \includegraphics[width=1.0\linewidth]{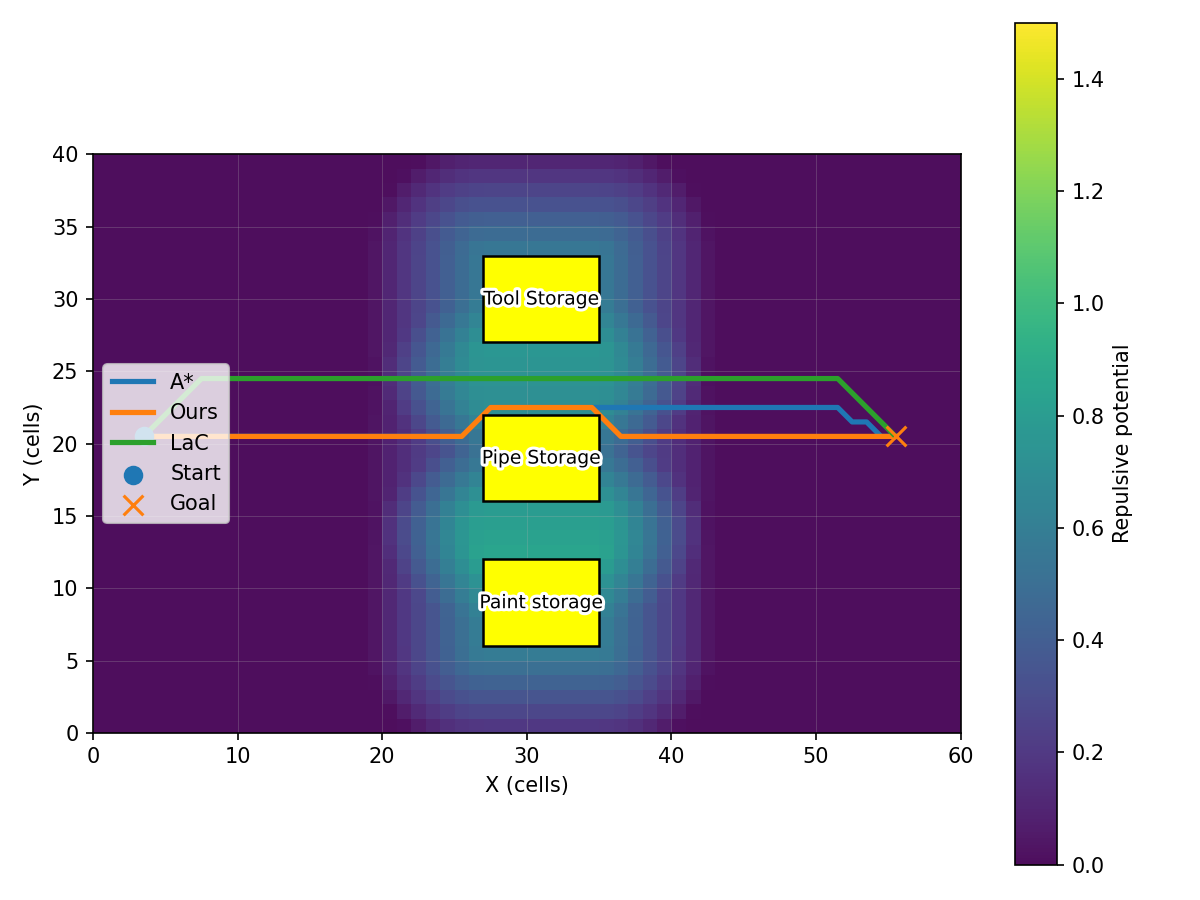}
    \caption{Prompt notifying the robot that the plumbing team has left, implicitly implying that the plumbing storage unit might be inactive}
    \label{fig:Plumbout_Results}
\end{figure}

\begin{table}
\setlength{\tabcolsep}{4pt}
\centering
\footnotesize
\caption{Bayesian Bootstrap Diagnostics and Path Metrics for Implicit Prompt with the Plumbing task has finished}
\label{tab:prompt3-plumbout}

\begin{tabular}{l *{3}{S[table-format=2.3]} cc}
\toprule
\multicolumn{6}{c}{\textbf{Hyperparameters} ($k{=}16$, $R{=}3000$, $\alpha=0.1$, $\tau{=} 1.0$, $\gamma = 1.0$)} \\
\midrule
& {\bfseries Paint Storage} & {\bfseries Pipe Storage} & {\bfseries Tool Storage} & & \\
\cmidrule(lr){2-4}
\textbf{Posterior CVaR} & 0.312 & 0.237 & 0.296 & & \\
\midrule
\multicolumn{6}{c}{\textbf{Path Metrics}} \\
\midrule
\textbf{Method} & \textbf{Length} & \textbf{Min.\ Dist.} & \textbf{Avg.\ Dist.} & \multicolumn{2}{c}{} \\
\midrule
A* path & 53.66 & 0.50 & 9.44 & \multicolumn{2}{c}{} \\
Ours    & 53.66 & 0.50 & 9.16 & \multicolumn{2}{c}{} \\
LaC & 55.31 & 2.50 &  10.72 & \multicolumn{2}{c}{} \\
\bottomrule
\end{tabular}
\end{table}

\subsection{Real-world evaluation through environment digital twin}
To demonstrate that our prompt-conditioned semantic method works in a real-world scenario, we used the building information model (BIM) of a lab space to test the experiment. The RVT file is first exported as an FBX and can be uploaded to any simulation environment, such as Gazebo or Unity. An example of this environment is shown in Figure~\ref{fig:BIM}. As shown in Figs.~\ref{fig:realroomempty} and \ref{fig:realroombusy}, our simulated environment closely mirrors the corresponding real-world lab rooms. One advantage of the RVT to FBX export is that the families within the model stay labeled with the same names that they have in the modeling software. We use a Unitree Go1 for the real-world experiment; the model will track its calculated path given the BIM. The path planning has been conducted by the backend and visualized in RViz, and the Unitree SDK has been assigned to follow the path. 
\par 
We evaluate the pathfinding algorithm using \emph{semantic success}: the fraction of trajectories that are semantically consistent with the environment and satisfy the prompt constraints, following prior work~\cite{LTLCODEGEN}. As shown in Figures~\ref{fig:empty_go1} and \ref{fig:busy_go1}, we evaluate two real-world scenarios corresponding to the prompts \textit{"Computer station 1 is empty"} and \textit{"Computer station 1 is active"}. Our results can be seen in Table \ref{tab:successRate}. Snapshots of the RViz trajectories for each of the tested scenarios can be seen in Figures \ref{fig:rvizactive} and \ref{fig:rvizinactive}.
\begin{figure}
    \centering
    \includegraphics[width=1.0\linewidth]{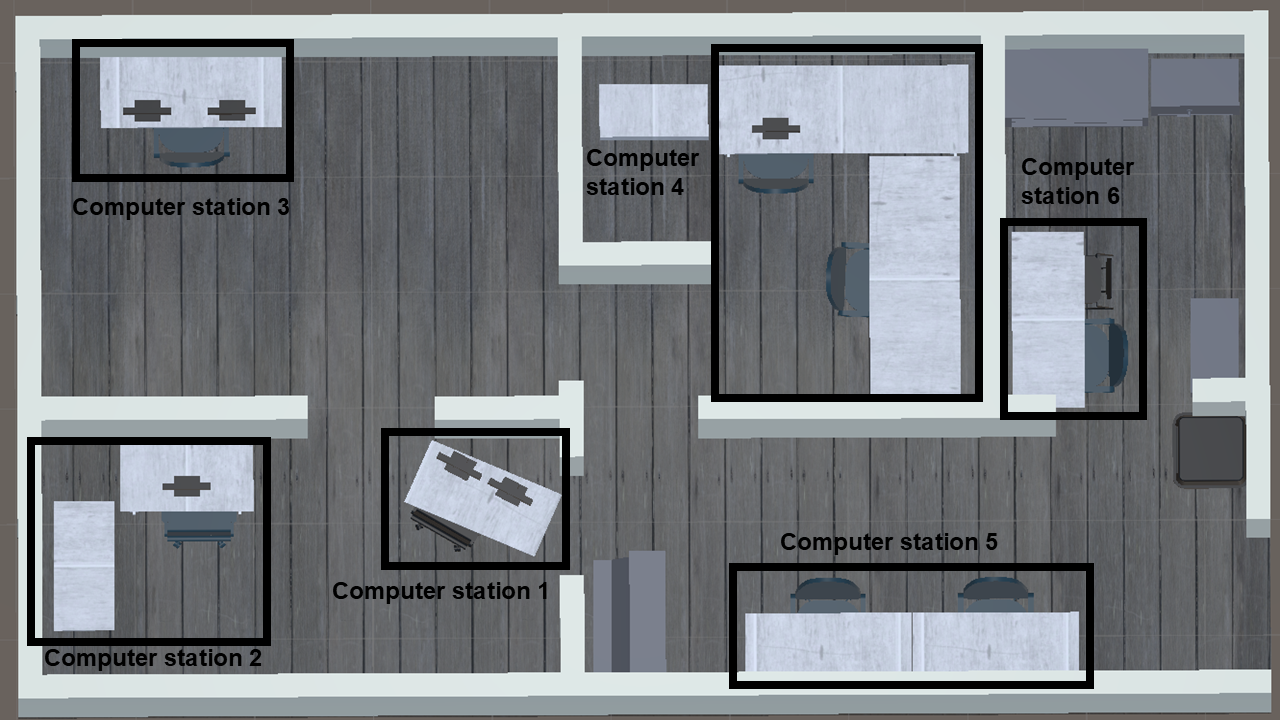}
    \caption{Imported BIM file to act as the semantic map ready for path planning.}
    \label{fig:BIM}
\end{figure}
\begin{figure}[t]
    \centering
    \includegraphics[width=0.9\linewidth]{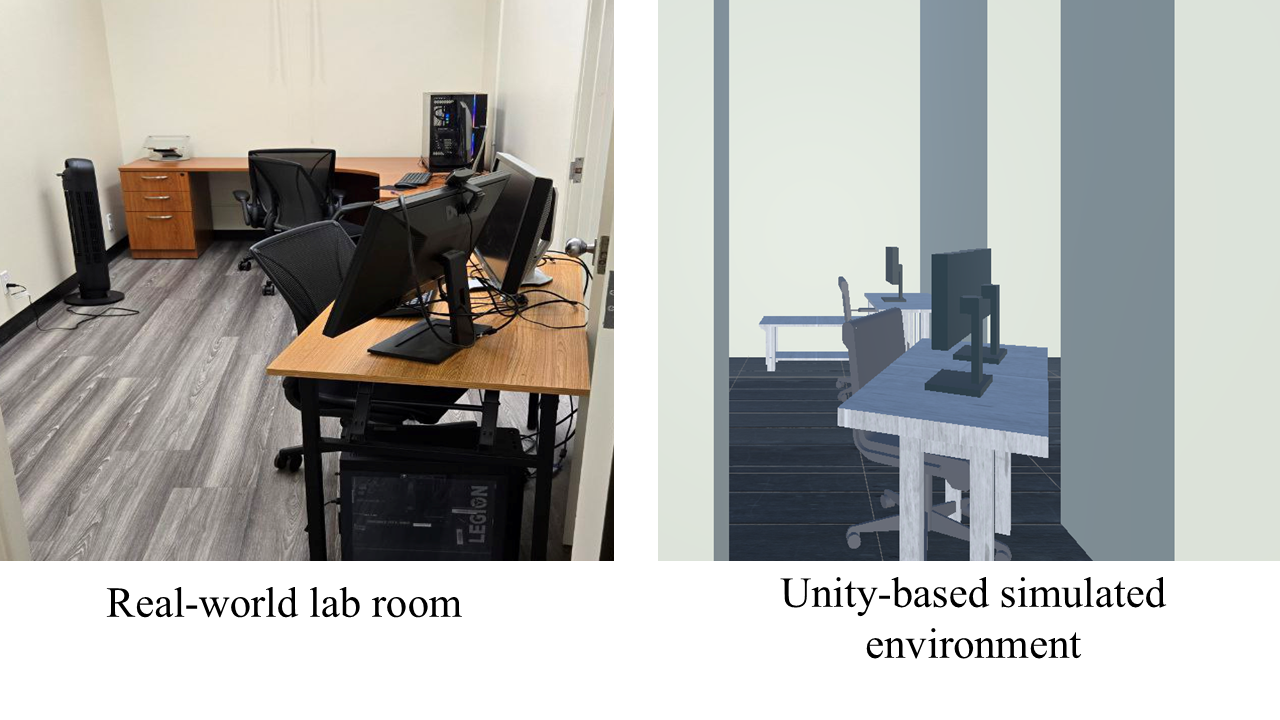}
    \caption{Comparison between the real-world lab environment (left) and its Unity-based simulation (right).}
    \label{fig:realroomempty}
\end{figure}

\begin{figure}[t]
    \centering
    \includegraphics[width=0.9\linewidth]{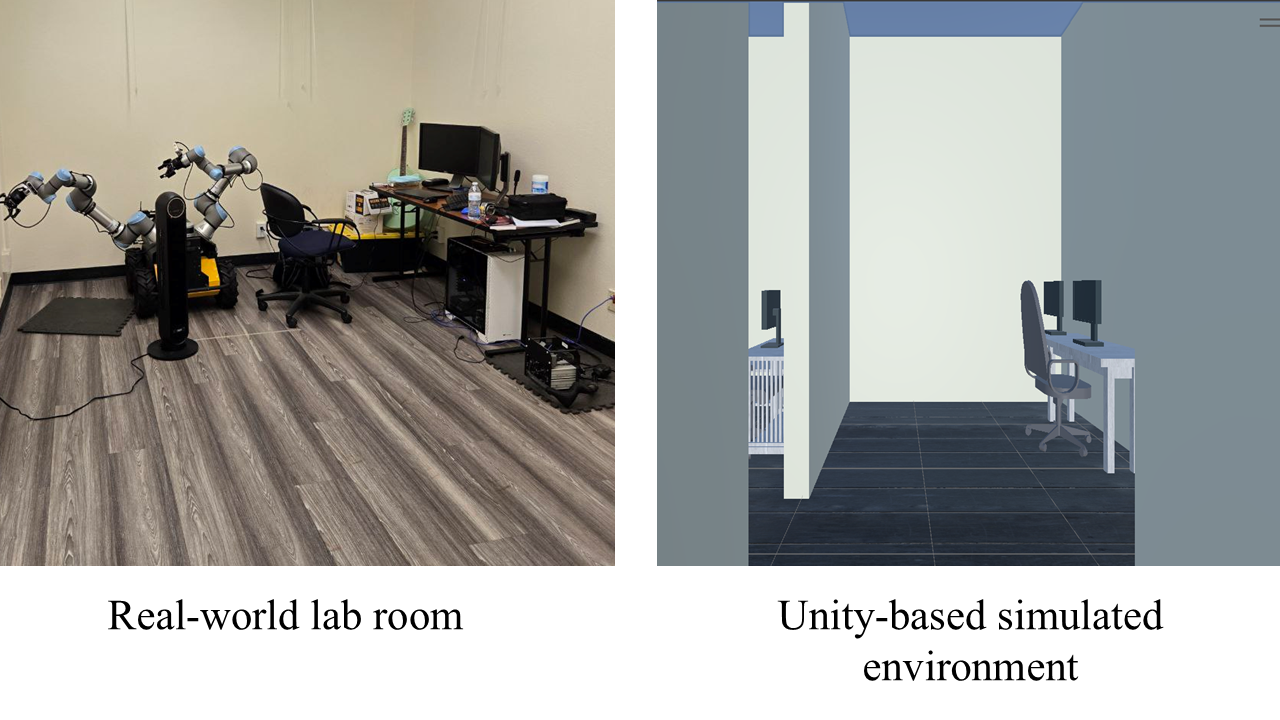}
    \caption{Comparison between the real-world lab environment (left) and its Unity-based simulation (right).}
    \label{fig:realroombusy}
\end{figure}

\begin{figure}[t]
    \centering
    \includegraphics[width=\linewidth]{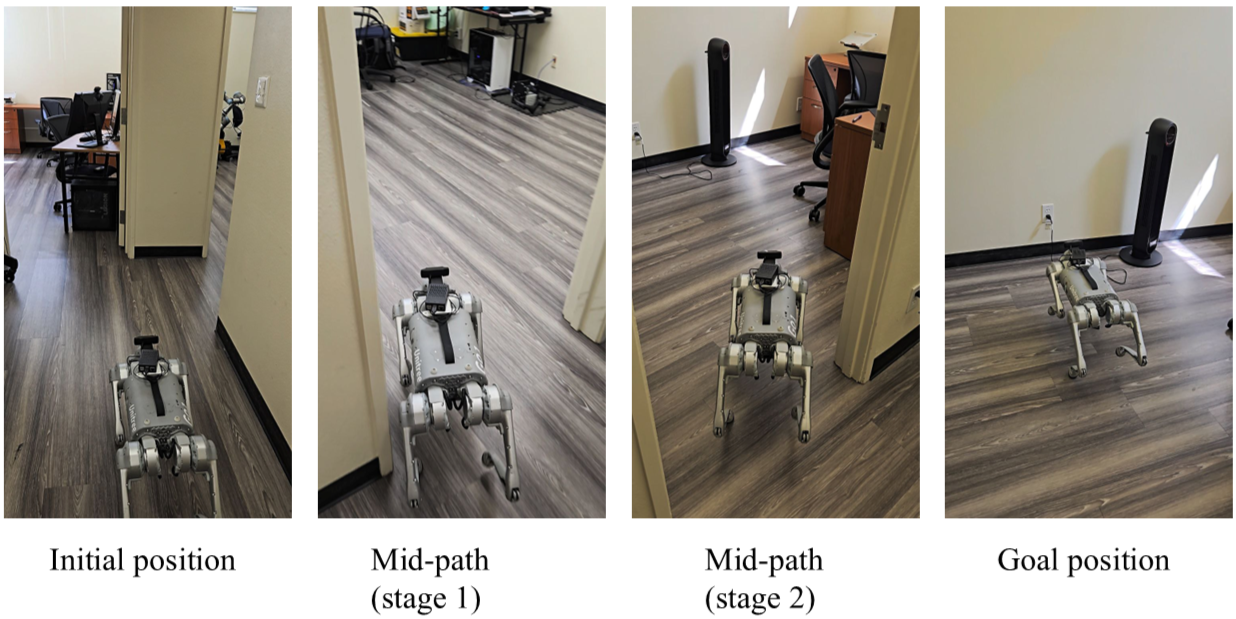}
    \caption{Snapshots of Go1 navigation given the prompt: \textit{"Computer station 1 is  active. Go to the destination."}}
    \label{fig:busy_go1}
\end{figure}

\begin{figure}[t]
    \centering
    \includegraphics[width=\linewidth]{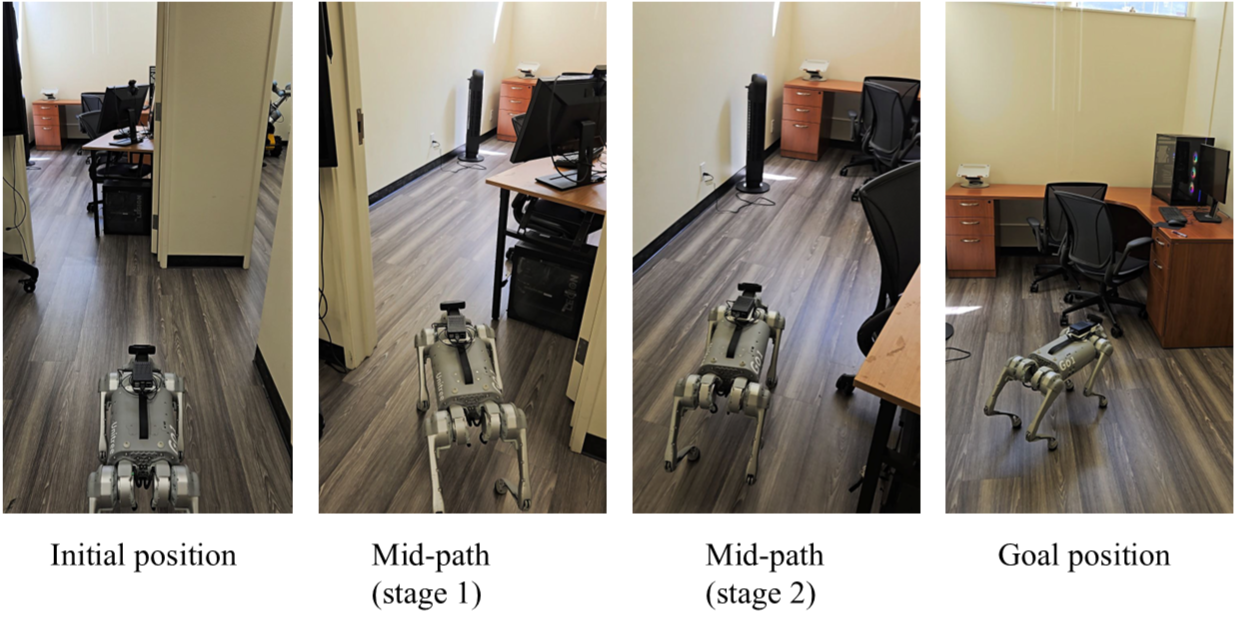}
    \caption{Snapshots of Go1 navigation given the prompt: \textit{"Computer station 1 is empty. Go to the destination."}}
    \label{fig:empty_go1}
\end{figure}
\begin{figure}
    \centering
    \includegraphics[width=1.0\linewidth]{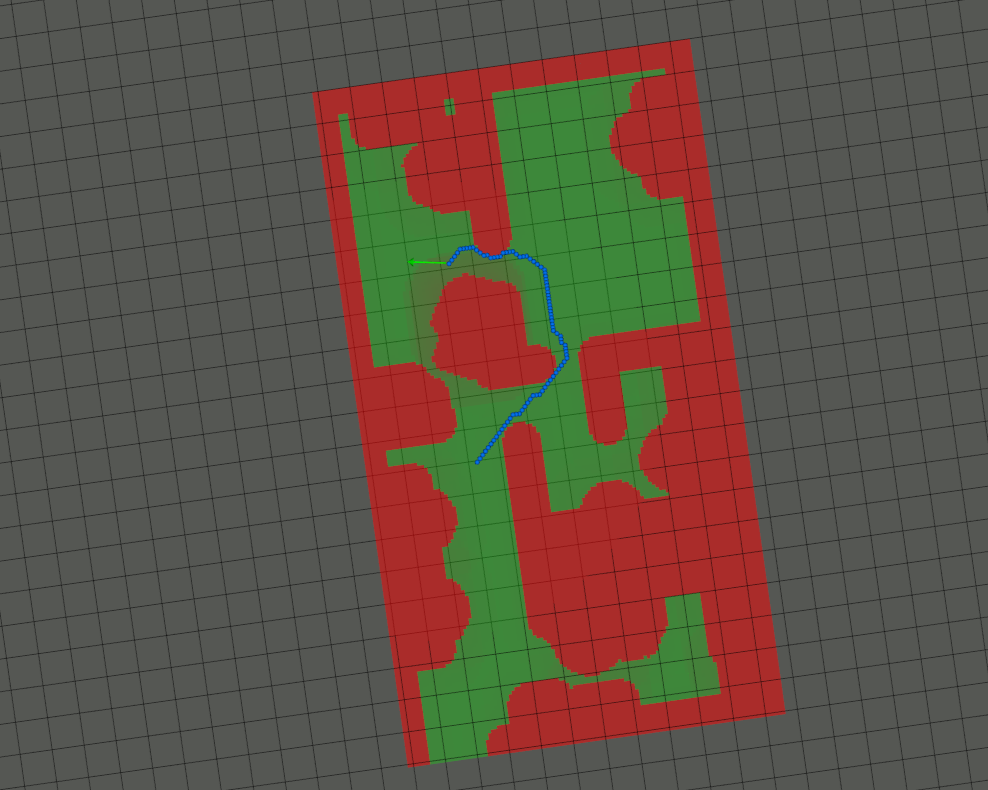}
    \caption{Path that was taken under a semantically active prompt}
    \label{fig:rvizactive}
\end{figure}
\begin{figure}
    \centering
    \includegraphics[width=1.0\linewidth]{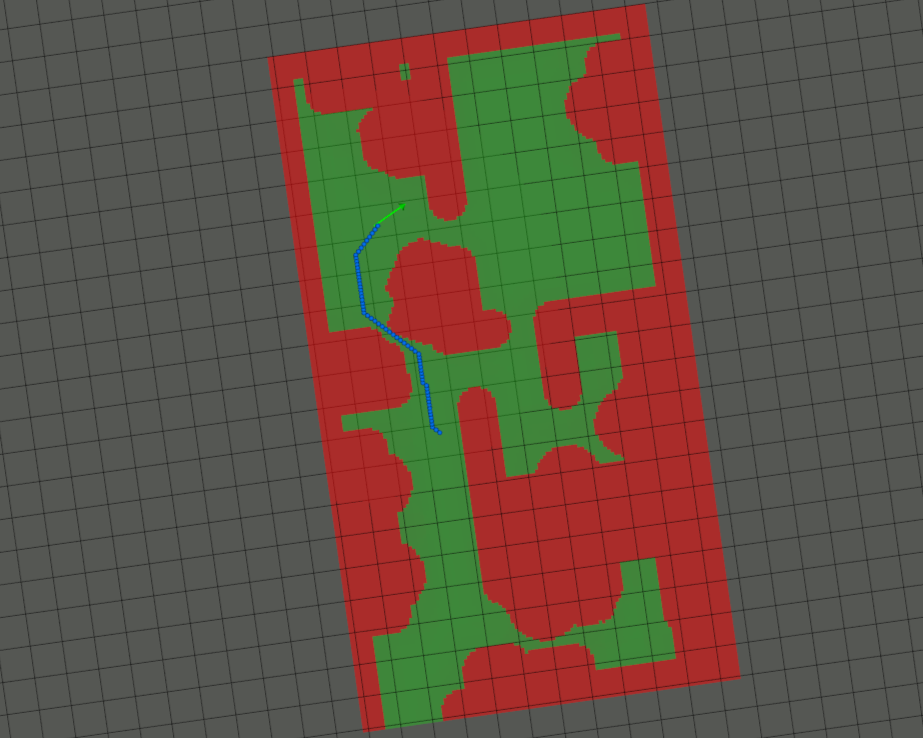}
    \caption{Path that was taken under a semantically inactive prompt}
    \label{fig:rvizinactive}
\end{figure}

\par
\begin{table}
\centering
\caption{Semantic Success Results in Real-World Experiments}
\begin{tabular}{|l|c|}
\hline
\textbf{Prompt} & \textbf{Semantic Success} \\ \hline
Computer station 1 is active. Go to the destination. & 100\% \\ \hline
Computer station 1 is empty. Go to the destination.  & 100\% \\ \hline
\end{tabular}
\label{tab:successRate}
\end{table}

\subsection{Ablation study}
LLMs are often time and energy-intensive, which can constrain how many queries a robot can issue in practice. We therefore study how the posterior statistics vary with the number of shots \(k\). As summarized in Table~\ref{tab:timing-shots-thin}, multi-shot querying need not be slow when the API aggregates concurrent calls, yielding near-constant per-shot latency up to moderate \(k\). Moreover, Figures~\ref{fig:Ablation_results_empty} and~\ref{fig:Ablation_results_busy} show that increasing \(k\) smooths variability in the posterior and captures a broader sample of the LLM’s stochastic judgments—leading to more stable, semantically consistent costs than treating the language model as a single-shot sensor.

\begin{table}
\caption{LLM sampling time vs.\ number of shots (10 runs each).}
\label{tab:timing-shots-thin}
\centering
\begin{tabular*}{\linewidth}{@{\extracolsep{\fill}} r r r r r}
\toprule
\textbf{k} & \textbf{runs} & \textbf{mean [s]} & \textbf{std [s]} & \textbf{mean/shot [s]} \\
\midrule
 1  & 10 & 0.9500  & 0.2742  & 0.9036 \\
 2  & 10 & 0.8703  & 0.0586  & 0.7963 \\
 4  & 10 & 1.1120  & 0.1901  & 0.8792 \\
 8  & 10 & 1.2432  & 0.3936  & 0.8219 \\
16  & 10 & 2.0986  & 0.3180  & 1.2037 \\
32  & 10 & 12.5938 & 26.3214 & 2.6507 \\
\bottomrule
\end{tabular*}
\end{table}

\begin{figure}
    \centering
    \includegraphics[width=1.0\linewidth]{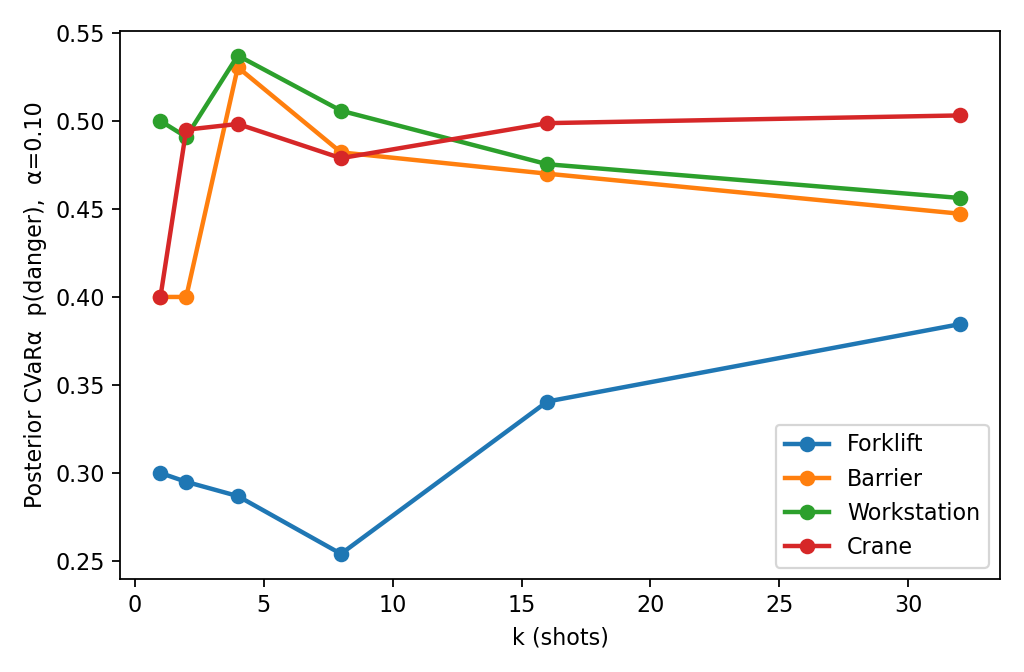}
    \caption{Ablation study over changing how many shots it takes using the CVaR given the prompt: "The forklift is off schedule, but everything else is busy." We can see general stability in the CVaR value over the different shots.}
    \label{fig:Ablation_results_empty}
\end{figure}

\begin{figure}
    \centering
    \includegraphics[width=1.0\linewidth]{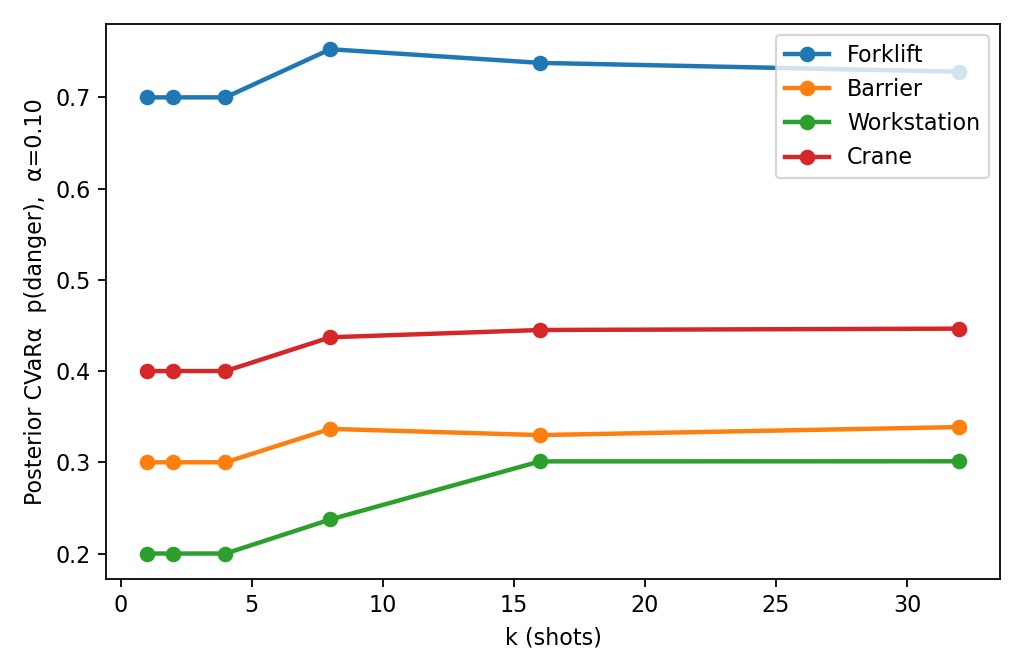}
    \caption{Ablation study over changing how many shots it takes using the CVaR given the prompt "The forklift is busy, but everything else is off schedule." We can see general stability in the CVaR value over the different shots.}
    \label{fig:Ablation_results_busy}
\end{figure}    
\section{Discussions and Limitations}
It is important to note that the LLM is not the main component of this algorithm. The LLM acts as an encoder that regresses Prompt \(P \to \mathbb{R}\), which is aiding us in injecting prompt semantics into classical path planning methods. At its core, this work attempts to add new interpretations to provide language as cost into closed-form settings by using NL encoders to quantify certain semantic qualities in user prompts. LLMs are a strong candidate for this task. However, depending on deployment and application specifics, more lightweight encoders can be deployed, which are tailored to the task requirements. 
\par
We use Bayesian bootstrapping to approximate a posterior over per-class semantic risk. We then compute \(\mathrm{CVaR}_\alpha\) of this posterior to obtain a tail-aware cost that can be plugged into classical search without changing the planner. Depending on the value of \(\alpha\), we can control the cost values to be closer to the mean \(\alpha\to0 \) or very risk-averse \(\alpha \to 1\). This way, we can generalize and modulate different behaviors by using the posterior as opposed to using point estimates from the sample set. Furthermore, cost thresholding can be introduced so that if the repulsive cost surpasses a threshold, path planning could return a no-feasible condition path in which alternative mission parameters can be considered. Future work will learn per-environment posteriors and danger metrics rather than class-specific ones to better capture interactions among classes.
\par

\section{Conclusion}
We introduced a nonparametric, language-conditioned risk model and a closed-form planner with optimality guarantees that fuses semantic information from both the environment and the user prompt, whether provided implicitly or explicitly, by casting the task as a classical path-planning problem. By shaping a CVaR-based semantic potential and solving with MHA* using a consistent Euclidean anchor, our method yields solutions optimal with respect to the combined cost and adapts its conservatism to the prompt: safe descriptions relax the semantic penalty and deliver efficiency gains (shorter paths, fewer expansions, lower latency) at unchanged success, while risk-heavy descriptions increase clearance with only modest extra cost. Across digital-twin and on-robot evaluations, this prompt adaptivity leads to higher flexibility and efficiency than previous works in both semantically safe and unsafe scenarios. More broadly, our results position natural language as a practical safety prior for fieldable planners, enabling task and context-specific behavior without retraining.
\FloatBarrier      
\bibliographystyle{IEEEtran}
\bibliography{bib}
\vspace{12pt}

\end{document}